\documentclass[11pt]{article}
\usepackage{tabularx}
\usepackage{amsmath, amsthm, amssymb}
\usepackage{algorithm}
\usepackage{algorithmic}
\usepackage{enumerate}
\usepackage{nccmath}
\usepackage{geometry}
\usepackage{bm}
\usepackage{multirow}
\usepackage{hhline}
\usepackage{xurl}
\usepackage{mathtools}

\newcommand{\ph}{\varphi}

\newcommand{\bb}{\mathbb}

\usepackage{enumerate}

\usepackage{array}
\newcolumntype{L}{>{\arraybackslash}m{7cm}}
\usepackage{tcolorbox}
\newcommand{\fon}[1]{\fontfamily{#1}\selectfont} 

\usepackage{wrapfig}

\newtheorem{statement}{Statement}

\newtheorem{theorem}{Theorem}

\usepackage{authblk}
\begin{document}
	\title{Theory of Machine Learning  with  Limited Data}
	
	\author{Marina Sapir} 

\affil{MetaPattern}
	\date{}
\maketitle

\begin{abstract}
	
	  Application of machine learning may be understood as  deriving  new knowledge for practical use  through  explaining accumulated observations, training set.   Peirce used the term  \textbf{abduction} for this kind of inference. Here I formalize the concept of abduction for real valued hypotheses, and show that 14 of the most popular textbook ML learners (every learner I tested),  covering classification, regression and clustering,  implement this concept of abduction inference. 
	The approach is proposed as an alternative to   Statistical learning theory, which requires an  impractical assumption of indefinitely increasing training set for its justification. 
\end{abstract}

\section*{Introduction}



The only commonly accepted theory of machine learning (ML) is statistical. It is not declared explicitly, but all results in this theory  implicitly assume   that (1) we are learning a dependence generated by a fixed probability distribution and   (2) the training set may be increased indefinitely so that  ``eventually'' this distribution will be well represented by the training set. 

For most of practical cases these assumptions are a stretch.  Applied ML is about learning a nondeterministic dependence   from a given finite sample for some urgent decision making.  Limited sample does not allow one to assume existence of  probabilities or even existence of infinite populations.  Essentially, applied ML and theoretical ML solve different problems.  As a result, theoreticians cannot answer the questions practitioners ask \cite{Papaya}.

Instead I propose to apply Peirce's 
\cite{Pierce} pragmatic view on learning from experimental data.  Within this paradigm, ML is a   search for the best  explanation of  the observations. This search is logically understood by Peirce as abduction inference.  In artificial intelligence, problems of diagnostics were already considered to be an example of abduction \cite{AbductionComplexity}.  

Here I   adapt the concept of abduction to deal with real-valued   hypotheses    and  formulate the concept of a abduction learner.  I  conjecture that every (worthy) learner in ML is abduction learner. Fourteen   popular textbook learners for classification, regression, clustering are shown to support the conjecture.

To the best of my knowledge, the proposed approach  is  the first one   to explain  and logically justify  large variety of  existing learners from a single point of view.  Pragmatic understanding of applied ML  opens a new path to solving  ``how to'' questions practitioners ask,  to design new useful learners for real life problems. 

Here is a brief description of each section: 
\begin{enumerate}
	\item  Traditional views on ML: learning to predict, statistical learning. 
	\item Pragmatic view on learning, informal description.  
	\item Logic of data explanations: alignments and deviations.
	\item Logic of data explanations: aggregation of deviations
	\item Logic of data explanations: recursive aggregation.
	\item  Explanation criteria 
	\item Abduction learning procedure. The Main Conjecture. 
	\item Proof that   the most popular learners are abduction learners. 
	\item Conclusions. 
\end{enumerate}

 \section{Traditional  views on ML} 
 
 Here I describe two understandings of  ML: traditional one and  one of statistical learning theory.  
 
 \subsection{Givens}
 
 Denote $\Omega$ the set of real life objects of interest.  For example, this  may be patients with skin cancer, or bank clients or  engine failures. There is a  hidden  essential quality we would like to find out (may be, a diagnosis or prognosis).
 Some properties (\textbf{features}) of the objects  $\Omega$   can be always evaluated and numerically expressed. Some of them are expected to be relevant to the hidden property.  Suppose, there are $n$ such features.  Denote $X \subseteq R^n$ domain of  feature vectors for objects in $\Omega$.   The hidden essential quality also has numerical expression values in  $Y \in R.$ The value of the hidden essence in a given object  is called ``\textbf{feedback}''.  We assume there is an ``\textbf{underlying dependence}''  $\varphi: X \rightarrow Y$ between feature vectors  and the feedback. Yet, we can not assume that the dependence is deterministic.  
 
 For example,  the features may not define completely  the feedback we are trying to model, there is uncertainty in measurements, random mis-classification and so on. Objects with the same features may have different feedback, and the same object evaluated twice may have different features or even feedback.  
 
 This is not a bad luck, but a inevitability.  Indeed, ML is needed only when there is no exact theory explaining the phenomenon we are trying to predict.  Therefore, we do not know what it depends on. The measurements have intrinsic uncertainty. 
 
 The information about the underlying dependence $\ph$  is given as training set:  observations about values of  feedback  in certain data points 
 set of tuples $\{\langle   x, y \rangle\}.$  These tuples will be also called empirical instances. 
 
  The sooner we find the proper hypothesis about the dependence, the better for the decision making. So,  data shortage is not a bug, it is a feature of ML.

\subsection{Prediction problem} 

The goal is assumed to be prediction of future values of the nondeterministic dependence.

  For example, here is how the prediction problem is understood in \cite{Stability}: Given a training set $S$ and data point $x$  of a new observation  $\langle x,  ?  \rangle$ predict its feedback $y.$
  
  The main issue with this problem statement is that to evaluate the decision and to select between the hypotheses we need to know what is not given: the future. 

There is no way to solve this problem with the available data. 


\subsection{Statistical Learning Theory Approach}

Statistical Learning  (SL) theory is  the only commonly accepted theoretical approach to ML. 



	This is how the  proponents of the SL theory   understand the problem: 	``Intuitively, it seems reasonable to request that a learning algorithm, when presented more and more training examples, should eventually  ``converge” to an optimal solution." \cite{StatTheory} 
The ``optimal solution'' here is the hypothesis having  the ER loss criterion close to minimal for the given class of functions regardless of the distribution.

 It does not appear to be intuitive to solve the problem where more and more training examples are expected, if we have only one finite training set.
V.  Vapnik \cite{VapnikBook} formulated the justification of the statistical approach  in the most direct way 

\begin{quotation}
	\textit{	Why do we need an asymptotic theory $\langle \cdots \rangle$ if the goal is to construct algorithms from a limited number of observations?
		The answer is as follows:
		To construct any theory one has to use some concepts in terms of which the theory is developed $\langle \cdots \rangle.$ }
\end{quotation}

In other words, statistical learning theory  assumes  indefinite increase of the training set, so that the statistical approach can prove some results. Statistics  has laws of large numbers, so the problem has to be about ever increasing training sets and  convergence. 


The  apparatus of probability theory does not help to understand  and solve the true pragmatic problem with fixed finite data, finite time allocated for decision making and un-quantifiable  uncertainty.

\section{Pragmatic view on learning} 

First,  I address the common misconception  that ML is induction.  


\subsection{ML is not induction, but abduction }

It is a common belief that ML is an induction inference. But is it? 

The dictionary says that induction is a method of reasoning from a part to a whole, from particulars to generals, or from  an individual to universal.  Roughly,  induction extends a property of a part  on  the whole.  We start with  objects of a certain class and their known common  property,  then  we conclude that all objects of this class have the same property.  This  procedure is exactly opposite to deduction, when knowing a property of a class and one object of the class, we infer the same property of the object. For both these types of inference, the   common property of objects of interest (the hypothesis),  is given in the beginning. 
 
This is not how we learn. We start with objects (observations) and no hypothesis, The property - the pattern  - which unites objects is not known. It is what needs to be found. 
  
Guessing the hypothesis to explain the facts is what Peirce called abduction inference \cite{PeirceV1}. 

\subsection{Informal description of pragmatic learning problem }
From pragmatic point of view, the goal of learning  is to find the best explanation of observations, not prediction. Let $h$ be a function used to explain observations. I call it explanation hypothesis, or simply explanation. Notation $S$ is used for training set, $H(h)$ set of all hypothetical instances of the function $h$. We will be interested in the conglomerate  of instances $M(h, S) =  \bigcup \{S, H(h)\}. $ 

Informally, we can formulate the \textbf{explanation principle}:
\begin{quote} \label{informal} 
	\textit{For an explanation $h$ to be any good, the values of feedback on instances  in $M(h, S)$ with close data points  shall be close. }
\end{quote}

Closeness of feedback on close data points in the conglomerate $M(h, S)$ can be used to evaluate  an explanation quality.   

The explanation principle makes it clear that a nondeterministic dependence has to have mostly close feedback on close data points to be ``explainable''. And to be a good explanation for an explainable dependence, a hypothesis has to have a close feedback on close data points as well.

\section{Data Explanation Logic}

Now,  I concentrate on developing a formalism for the abduction  criterion 


%
%

\subsection{Language of  Alignment }

So far, I considered only  an underlying dependence with a single variable. In a general case  the restriction is not necessary.  For example, using two independent variables may be convenient for formalizing ranking problem. 

The first order signature has these  6 + $n$   sorts among others:

\begin{table}[H]
	\begin{center}
		\caption{Sorts of LA}
		\label{tab:table1}
		\begin{tabular}{| l | l| l | l|}
			\hline
			Sort & Content & Variables & Constants \\ \hline
			 $\bb{N}$  & Natural numbers &  $ i, i, k, l, m, i_1, \ldots $ & $n, 0, 1$\\
			 $  \bb{X}_1, \ldots,  \bb{X}_n$ & Domains of the independent variables  & $x, x_1, x_i^j,  \ldots$ &\\
			 $\bb{Y}$ & Domain of the feedback   & $y, y_1,  \ldots$ & \\ 
			 $\bb{H}$ &  Types of observations & $s, s_1, \ldots$  & $\asymp,  \asymp_1,  \ldots$\\ 
			 $\bb{O}$ &  Types of hypothetical instances & $s, s_1, \ldots$ & $\approx, \approx_1, \ldots$\\ 
			$\Psi$ & Instances  & $ \alpha, \beta, \alpha_1, \beta_1, \ldots$  & \\ 
		   $\bb{R}$ & Real numbers & $r, r_1, r_2\ldots$ & \\

			\hline
		\end{tabular}
	\end{center}
\end{table}

Let us notice that observations may be of different types. It may be convenient for identification of censored data, for example. The set $\bb{O}$ combines symbols for all types of observations.

Some other sorts will be introduced later, when we need them to describe aggregation  and regularization.

Usually,  domains of the sorts $ \bb{X}_1, \ldots,  \bb{X}_n$  and  $\bb{Y}$ are some metric spaces. However,  the triangle axiom  of a distance on a domain is irrelevant,  not required.  

%
%
%
%

The first order symbols are defined in the Table \ref{tab:Functions}.

\begin{table}[h!] 
	\begin{center}
		\caption{Function on instances}
		\label{tab:Functions}
		
		\begin{tabular}{|l | l | l | l | l |  } 
			\hline
			& Symbol   & Arity & Sorts  & Semantic   \\ \hline
			1 & $\bm{x}$  & 2  & $\Psi  \times \bb{N} \rightarrow \bb{X}_i$ &  $\bm{x}(\alpha, i)$ is $i$-th     variable  of $\alpha \in \Psi$ \\
			2 & $\bm{y}$  & 1  & $\Psi  \rightarrow \bb{Y}$ &  $\bm{y}(\alpha)$ is feedback   of $\alpha$ \\
			3 &  $\bm{s}$ & 1  &  $\Psi \rightarrow \bb{H} \cup \bb{O}$  & $\bm{s}(\alpha) $  is type symbol of $\alpha$ \\
			4 & $\rho_x$ & 3  & $\Psi \times\Psi  \times \bb{N} \rightarrow \bb{R}$ &  $\rho_x(\alpha_1, \alpha_2, i)$ =  $\| \bm{x}(\alpha_1, i)  -  \bm{x}(\alpha_2, i) \|$\\
			5 & $\rho_y$ & 2  & $ \Psi  \times \Psi   \rightarrow \bb{R}$ & $ \rho_y(\alpha_1, \alpha_2) = \|\bm{y}(\alpha_1) - \bm{y}(\alpha_2) \|$  \\
			6& $\bm{o}$ & 1  & $ \Psi     \rightarrow  \{0,1\}$ & $  (\bm{o}(\alpha) =1 ) \leftrightarrow  (\bm{s}(\alpha) \in \bb{O}) $ \\
					7& $\bm{h}$ & 1  & $ \Psi     \rightarrow  \{0,1\}$ & $ (\bm{h}(\alpha) =1 ) \leftrightarrow  (\bm{s}(\alpha) \in \bb{H}) $ \\
	
			\hline
		\end{tabular} 
	\end{center}
\end{table}
\bigskip
\hskip1pt

 The function $\bm{o}$ distinguishes observations from the hypothetical instances. It may be useful, when there are many types of observations.

When there is only one independent variable ($n= 1$), I will skip the index variable in $\rho_x$ and $\bm{x}$.

\subsection{Alignments and Deviations}

The  formulas are formed as in first order predicate logic. 

\begin{enumerate}
	\item 	Any first order predicate  $\pi(\alpha, \beta, i): \alpha, \beta \in \Psi, 1 \le i \le n$ is an \textbf{alignment} if 
	\begin{multline*}
	 \forall \alpha  \; \forall \beta  \; \forall \alpha_1  \; \forall \alpha_2  \\
		\Big( \pi(\alpha, \beta)  \&  \big( s(\alpha) = s(\alpha_1) \big) \& 
		\big( s(\beta) = s(\beta_1) \big) \& \big( \rho_x(\alpha_1, \beta_1, i) \leq 
		\rho_x(\alpha,  \beta, i) \big) \Big) \\ \Rightarrow \pi(\alpha_1, \beta_1, i). 
	\end{multline*}



	\item For a pair of first order formulas $\alpha_1, \alpha_2$  and a natural $i, \; 1 \le i \le n$  their \textbf{ deviation} $\delta(\alpha_1, \alpha_2, i)$ is defined as 
	$$ \delta(\alpha_1, \alpha_2, i) = t(\rho_x(\alpha_1, \alpha_2, i), \rho_y(\alpha_1, \alpha_2)), $$
	where $t: \bb{R}^{\ge 0} \times \bb{R}^{\ge 0} \rightarrow \bb{R}^{\ge 0} $  is antitone  by $r_1$ and  isotone  by $r_2.$ 
\end{enumerate} 


Let us consider an example. The simplest alignment  for the case when $n = 1$ is called \textit{point-wise:}
\begin{equation} \label{pwa}
\pi_{pw}(\alpha_1, \alpha_2) = (  \bm{s}(\alpha_1)  = \approx) \; \&  \;(\bm{s}(\alpha_2) = \asymp)  \; \& \;  ( \bm{x}(\alpha_1) = \bm{x}(\alpha_2 ) ).	
\end{equation}

The simplest deviation is 
\begin{equation}\label{pwd}
	\delta(\alpha_1, \alpha_2) =  \rho_y(\alpha_1, \alpha_2) = \|\bm{y}(\alpha_1) - \bm{y}(\alpha_2)\|.
\end{equation}

\section{Aggregation of deviations} 

To compare explanations, one has to compare  collections of deviations. Usually, it is done by mapping such a collection into a single  number for a comparison. 

Let us assume for simplicity that deviations corresponding to the same alignment are ordered one way or another.  

Denote $\bb{Z}$ an additional  domain of finite sequences of real numbers.  The elements are interpreted as sequences of deviation  for aggregation and will be denoted by capital letters $A, B, G.$  The elements of  sequences will be denoted  with corresponding small letters with indices.

Let $A^\prime$ be an operation of ordering elements of the sequence $A$ from smallest to largest. 

Here are the binary relations on this domain. 
\begin{table}[h!]
	\begin{center}
		\caption{Binary Relations}		
		\begin{tabular}{|l | l | l |  }
			\hline
			& Symbol      & Semantic   \\ \hline
1 & $A = B $ &   $ \forall i \;  a_i = b_i $\\
2 &$ A \sim B$ & $  A^\prime = B^\prime$\\
3 & $A < B  $ &   $(  \|A\| = \|B\| ) \&  \Big( \exists q \; \forall i \; \exists j \; \big( b_i = q(a_j) \big)  \& \big( b_i >  a_j \big)\; \Big).  $\\
4 & $A \le B  $ &   $(  \|A\| = \|B\| ) \&  \Big( \exists q \; \forall i \; \exists j \; \big( b_i = q(a_j) \big)  \& \big( b_i \ge a_j \big)\; \Big).  $\\ \hline
\end{tabular}
\end{center}
\end{table}

We will use three functions
\begin{enumerate}
	\item $\|A\|,$  length of $A$
	\item $\{A, b\}, $ adding number $b$ in the end of the sequence $A.$
	\item $g(A, i) =  $ $a_i.$
\end{enumerate}

An operation $\Omega:$  $\bb{Z} \rightarrow \bb{R}$will  be called \textbf{aggregation} if it satisfies the next 
 \textbf{axioms}:
\begin{enumerate}  \label{AggAxioms}
	\item \emph{\textbf{Order insensitivity:}} For any $A, B$ if $A \sim B$ then $\Omega(A) = \Omega(B)$
	\item \emph{\textbf{Monotony: }} For any $A, B$
	$$A \le B \Rightarrow \Omega(A) \le \Omega(B) \; \&$$
	$$A <  B \Rightarrow \Omega(A) < \Omega(B) $$
\end{enumerate}

We will call aggregation \textbf{stable}, if the next property is also satisfied:   
$$ \forall A \left( \underset{  i \le \|A\|} {\forall i}  \;  a_i = a_1 \right)  \Rightarrow \; \big(\Omega(A) = a_1\big). $$

\begin{statement} \label{State1}
	For any stable aggregation $\Omega(A)$  
	$min(A) \leq \Omega(A) \leq max(A).$
\end{statement}
\begin{proof}
Suppose $\exists A \; \Omega(A) > max(A).$
 If every element of the  sequence $B$ is $ \Omega(A)$ then $B > A.$  By the Monotony, $\Omega(B) > \Omega(A).$  This contradicts stability.
  The proof for the $minimum$ is similar. 
\end{proof}

\begin{statement}\label{State2} If $A < B$ then  $A^\prime < B^\prime$ and for any $i \;\; b^\prime_i > a^\prime_i$. 
	\end{statement}
\begin{proof}
	 It is obvious that If $A < B$ then  $A^\prime < B^\prime.$ 
	  $$\exists q \; \forall j \; \exists i \; \big( b^\prime_j = q(a^\prime_i) \big)  \&   \big( b^\prime_j > a^\prime_i \big).$$ 
	 
	Take any $i$ and denote $b^\prime_j = q(a^\prime_i)$. 
	Then $b^\prime_i > a^\prime_i$ in all three possible cases: 
	\begin{enumerate}
	\item If $i = j,$ it is obvious. 
	
	\item If $i > j$, $a^\prime_i < b^\prime_j  \le  b^\prime_i.$ So, $a^\prime_i  <  b^\prime_i.$
	
	\item If $i < j$, there exists $k$ such that $k  >  i$ and $q(a^\prime_k) = b_l, \; l  <  i.$ Then $a^\prime_i \le a_k < b^\prime_l \le b^\prime_i. $ 
	\end{enumerate} 
\end{proof}

Here is an example of an aggregation operation. 
 \begin{theorem}\label{PercentileTheorem}
Any percentile  is a stable  aggregation.
\end{theorem} 
\begin{proof}
	
	Denote $P_r(A)$ $r$-th percentile of $A \in \bb{Z}.$
Order 	insensitivity  is obvious, because percentile does not take into account order of elements.
	
Let us prove monotony. Let $A < B.$ From the Statement \ref{State2} it follows that $A^\prime < B^\prime$ and for any $i \;\; a^\prime_i < b^\prime_i.$

Denote $m = \|A\| = \|B\|.$ If $p= (r/100  \cdot m) $ is integer,  then $\mu(A) = a_{p + 1},$  $\mu(B) = b_{p+1},$  and $\mu(A) < \mu(B). $
If $p$ is not integer, denote $q = \lfloor p \rfloor,$   $ \mu(A) = ( a_q + a_ {q+1})/2 ,$  $\mu(B)  = ( b_q + b_ {q+1})/2 $ and $\mu(A) < \mu(B). $

Stability is trivial.. 
\end{proof}

 \section{Recursive aggregation} 
 
 Here I introduce a language to generate a class of aggregation functions called recursive aggregation.

\subsection{Language of recursive aggregation}


We will use additional functions

%

\begin{table}[h!]
	\begin{center}
		\caption{Function symbols}
		
		\begin{tabular}{|l | l | l | l | l |  }
			\hline
			& Symbol   & Arity & Sorts  of arguments  & Semantic   \\ \hline
			1 & [ ]& 1   & $\bb{R}  \rightarrow \bb{R} $ & scaling    \\
	
			3 &$ \oplus$ & 2  & $\bb{R} \times \bb{R} \rightarrow \bb{R}$ &  compounding \\
			4 & $\sum$  & 2  &  $\bb{G} \times \bb{N}\rightarrow \bb{R}$  &  recursive aggregation  \\
		    6 & $\eta$ & 2& $ \bb{R}  \times \bb{N} \rightarrow \bb{R}$ &  normalization   \\
			\hline			
		\end{tabular} 
	\end{center}
\end{table}

\subsection{Axioms of recursive aggregation}
\subsubsection{Function $[\;\; ]$ (scaling)}

The function $[\;]$ is strictly monotone:
	$$x_1 > x \Rightarrow  [x_1]>  [x]$$

Typical examples of the function $[ ]$
\begin{itemize}
	\item $[x] = x$
	\item $[x]  = x^2$
\end{itemize}
The  function	$max(x- a, 0)$ cannot be used for scaling since it is not strictly monotone. 

\subsubsection{Function $\oplus$ (compounding)}
 The function $\oplus$ has three axioms:  
 \begin{align*}
 Simmetricity: \;\; & x \oplus y=  y \oplus x. \\
Monotony: \;\; & ( x_1  > x ) \; \&  \;( y_1 > y) \Rightarrow (x_1 \oplus y_1)   >  (x \oplus y ) \\
Associativity: \;\;  & (x \oplus (y \oplus z) )  = ( (x \oplus y) \oplus z)
 \end{align*}

The next functions satisfy all the axioms:
 \begin{itemize}
 	\item $x \oplus y = x + y$
 	\item $x \oplus y = x \cdot y$
 	\item $x \oplus y= max(x, y).$
 \end{itemize}
 
\subsubsection{Function $\eta$ (normalization) }
The function $\eta$ is strictly isotone by the first variable and antitone by the second variable. 

$$ (x_1 > x) \Rightarrow ( \eta(x_1, n) > \eta(x, n) )  \&  $$
$$(n_1 > n) \Rightarrow ( \eta(x , n_1) \leq   \eta(x , n)). $$

Typical examples of the function $\eta(t, n)$
\begin{itemize}
	\item $\eta(x, n) = x/n$
	\item $\eta(x, n) = x^{-n}.$
	\item $\eta(x, n) = x.$
\end{itemize}

\subsubsection{Recursive aggregation function $\sum$}

The function $\sum$ is defined recursively 
\begin{align*}
 Step \; 1:\;\; & \bm{\sum}(A, 1) = [a_1]\\
Step\;  i+1: \;\; &  \bm{\sum}(A, i + 1) = \bm{\sum}(A, i) \oplus [a_{i+1}].
\end{align*}

\subsubsection{Recursive aggregation}

Recursive aggregation  is defined by formula  $${\Omega}(A)  = \eta \Big( \bm{\sum}(A, \|A\|), \;  \|A\| \Big).$$

	 \begin{theorem}
	 	Recursive aggregation $\Omega(A)$ is an  aggregation. 
	\end{theorem}
		 \begin{proof}
	 
Let us prove order insensitivity. 
 	Suppose, a sequence $B$ is a permutations of a sequence $A$. 
 	
 	Each permutation can be obtained  by finite number of transpositions of neighboring elements. Suppose, $B$ can be  obtained from $A$ by $K$ transpositions. Let us prove the theorem by  induction over $K$. First, suppose $K = 1.$

 	Suppose, the  $B$ transposes elements $a_i, a_{i+1}$. 
 
 	Denote $\alpha(l), \beta(l)$ values of the recursive aggregation function  obtained on the step $l$ with the sequences $A,  B$ respectively. Since all the elements prior to $i$ are identical in these orders, $a(i-1) = b(i-1).$ By definition 
 	\begin{align*} 
 	   \alpha(i)  & =   \alpha(i-1) \oplus [a_i]   \\
 	   \alpha(i+1) & =  \alpha(i) \oplus [a_{i+1}] \\
 	   & = (\alpha(i-1) \oplus  [ a_i ]) \oplus [a_{i+1}] \\
 	  \beta(i) &=  \alpha(i-1)\; \oplus \;[a_{i+1}] \\
 	    \beta(i + 1) & =  (\alpha(i-1) \oplus [a_{i+1}]\;) \oplus[a_i ]\\ 
 	\end{align*}
Using symmetry and associativity of the function $\oplus$  we get $a(i+1) = b(i+1).$ All the elements starting from $i+2$ are identical in both sequences $A, B.$ Therefore $\Omega(A) = \Omega(B)$.

Suppose, we proved the property for $K = k$. Let us prove it for $K = k + 1.$ Suppose, the transpositions are ordered by the indices of involved elements,   and the last transposition involves elements $a_i, a_{i+1}.$ Then, the same considerations apply again. 

 Let us prove monotony. 
Suppose, $A < B$. According to the Statement \ref{State2}, then
$A^\prime < B^\prime$ and for any $i$ $a_i < b_i.$ 

Let us prove it by induction by $n= \|A\| = \|B\|.$ For $n = 1$ 
$$\Omega(A) =  \eta(\sum(a_1, 1),1)  = \eta([a_1], 1).$$
$$\Omega(B) =  \eta([b_1], 1).$$
Then $\Omega(A) < \Omega(B)$ because both functions $\eta, [\;]$ are strictly monotone by the first variable. 

Suppose, the statement is proven for $n = k.$ Denote $A^k, B^k$  sequences with the first $k$ elements of the $A, B$ respectively.  By the inductive hypothesis
\begin{align*}
\Omega(A^k, k) & < \Omega(B^k, k)\\
\eta( \sum(A^k, k) , \; k) & < \eta( \sum(B^k, k) , \; k).
\end{align*}

Since the function $\eta$ is strictly isotone by the first variable. it means
$$\sum(A^k, k)  < \sum(B^k, k)$$.

Let us prove monotony of operation $\Omega$ for $n = k + 1.$
\begin{align*}
\Omega(A) &= \eta(\sum(A, k+1), k+1) \\
     & = \eta(\sum(A^k, k) \oplus [ a_{k+1} ],  k+1)\\
 \Omega(B) & = \eta(\sum(B^k, k) \oplus [ b_{k+1} ],  k+1).
\end{align*}
First, notice that 
$$ \sum(A^k, k) \oplus [ a_{k+1}] < \sum(B^k, k) \oplus [ b_{k+1} ],$$
because both arguments of the operation $\oplus$ on the left are smaller than corresponding arguments on the right. 
The function $\eta$ is strictly monotone by the first argument. 
This proves the theorem.	
\end{proof}

Here are the most popular aggregation functions.
\begin{theorem} \label{AggregationsTheorem} For a sequence $A:  m = \|A\|,$
	\begin{align*}
	 L_1(A)  =  \frac{1}{m} \sum_i a_i \\
     L_2(A) =  \frac{1}{m} \sqrt{\sum_i a_i^2}\\
	L_3(A) =  ( \prod_i a_i)^{1/ m}
	\end{align*}
		are stable recursive aggregations. 
\end{theorem}
\begin{proof} For every operation the functions of the Recursive aggregation language are defined in the next table:
\label{tab:aggregations}
	\begin{center}
	\begin{tabular}{| c |c |c| c |}
			\hline
	& $ x \oplus y$ & $ [x]$  & $\eta(x)$  \\ \hline
	 $L_1$ & $ x + y$ & $ x $ & $ x/ m  $ \\ \hline
	  $L_2$ & $x + y$ & $ x^2 $ & $\sqrt{ x/ m } $ \\ \hline
     $L_3$ & $ x \cdot y $ & $x$ & $ x^{ 1/m} $ \\ \hline
	\end{tabular}
	\bigskip	
	 \end{center}	
If  for any $i: a_i = a_1 = a$< then $L_1(A) = L_2(A) = L_3(A) = a.$
This proves stability and the theorem. 
\end{proof}

\section{Explanation criteria} 

In general. an explanation criterion  evaluates quality of a hypothesis. Informally, the idea of such a criterion is presented in \ref{informal}. 

\subsection{Regularization}

As I mentioned in \ref{informal},  \textit{``For an explanation $h$ to be any good, the values of feedback on instances  in $M(h, S)$ with close data points  shall be close.''}  

In particular,  an explaining hypothesis $h$  itself shall have close feedback on close data points.  In means it shall not have high derivatives, when it is differentiable. 

To take into account derivatives, the model of Data Explanation Logic shall have yet another base with the only vector: a sequences of parameters of the explaining hypothesis. 

Regularization is a real-valued  function on this vector. Usually, the function is some evaluation of derivatives.

\subsection{Badness rule}

A badness rule is a triple of alignment criterion $\pi$, deviation function  $\delta$  and an aggregation operation $\Omega.$ For a recursive aggregation, the aggregation operation is further determined by the functions $\oplus, [\; ], \eta.$

Let us consider an    example of a badness  rule  called \textbf{Point-Wise rule}, $T_{pw}$.  The rule includes 
\begin{enumerate} 
	\item  An alignment relation (see (\ref{pwa})):
	\begin{align*}
		\pi_{pw}(\alpha_1, \alpha_2)  = \Big( ( \bm{s}(\alpha_1) \; = \; \asymp)  \; \& \; ( \bm{s}(\alpha_2) \; = \; \approx)  \; \& \; ( \bm{x}(\alpha_1) = \bm{x}(\alpha_2) \Big)
	\end{align*}
\item  A deviation function (see (\ref{pwd})) $$ \sigma(\alpha_1, \alpha_2)  = \rho_y(\alpha_1, \alpha_2). $$
\item  A recursive aggregation operation ``averaging''  (see Theorem \ref{AggregationsTheorem})  with operations
\begin{align*}
	x \oplus y & = x + y\\
	[x] & = x\\
	\eta(x, n)  & = \frac{x}{n}.
\end{align*}
.\end{enumerate}


\begin{theorem}\label{ERMtheorem}
If $\|y_1 - y_2\|  = |y_1 - y_2|$ then  $T_{pw}$  badness rule is equivalent with the empirical risk criterion
$$L(h, S) = \frac{1}{\|S\|} \sum_i |h( \bm{x}(\beta_i)) - \bm{y}(\beta_i)   |,$$
where the training set $S = \{\beta_1, \ldots, \beta_m\},$ $h$ is the explanation hypothesis.  
\end{theorem}
\begin{proof}
	If $\pi_{pw}(\beta_i, \alpha_i)$, then 
	$\alpha_i = \psi(\bm{x}(\beta_i), h(\bm{x}(\beta_i)), \asymp).$
	
	Deviation for an aligned pair $\alpha_i, \beta_i$ is 
	\begin{align*}
	\delta(\beta_i, \alpha_i)  &=    \rho_y(\alpha_i, \beta_i)  = |\bm{y}(\alpha_i) - \bm{y}(\beta_i)|\\
			\end{align*}
	
	Theorem \ref{AggregationsTheorem} shows that the recursive aggregation operation $L_1(A) = L(h, S)$ when 
	$$A = \{\delta(\beta_i, \alpha_i), i= 1, \ldots, m\}.$$
	
%
	\end{proof}

\subsection{Explanation criterion}

An explanation criterion consists of 
\begin{itemize}
	\item Series of badness rules.
	\item Regularization rule.
	\item Combining operation $C(A),$  defined on sequences of real numbers.  It maps outputs of all previous rules into a single number, criterion value.  Combining operation shall be monotone by all components of the vector $A.$
\end{itemize}

Regularization rule is not necessary. If there is only  one badness rule then combing operation is not used. 

\section{Search for the best explanation} 

Pragmatism implies a search for the best explanation. Here, I want to narrow a concept of search and describe two basic search procedures. 

Let us consider learning algorithm which minimizes an explanation  criterion $L(h, S)$ on a class of hypotheses $h \in F,$ given a training set $S$.  The next procedures will be considered standard. 

\begin{tcolorbox} [title =Basic training, fonttitle=\fon{pbk}\bfseries]
	  Basic training takes $F, S, L(h, S)$ and parameter $q,$  and it consists of  the next steps 
\hskip1em

\begin{itemize}
	\item\textbf{Focusing} (optional): transformation   $U: \; S \rightarrow  S_q$ 
	\item \textbf{Fitting}: generating  hypothesis $h \in F^\prime \subseteq F$   and evaluating of   $L( h, S_q) $ 
	\item  \textbf{Optimal selection}:  output of  an optimal explanation $h_q  = \arg \min_{F^\prime} L(h, S_q)$
\end{itemize} 
\end{tcolorbox}

Focusing may be a nonlinear transformation of the training set. 
Yet, typically, it is used to select observations, or features, or emphasize some of them  with weights. 

The selection of hypotheses may not go over whole class $F$, but its finite subclass $F^\prime \subseteq F.$

\begin{tcolorbox}[title = Wrapper strategy , fonttitle=\fon{pbk}\bfseries ] 
		Wrapper strategy takes $F, S, L(h, S),$ the empty set $Q$ and consists of the next steps:
\begin{itemize}
	\item \textbf{Wrapper loop}:  Generating  parameters $q ; \; $  $Q := Q \cup \{q\}$	
	\begin{itemize}
	\item \textbf{Basic training} with  parameters $q$ , outputs a  hypothesis $h_q$
	\item  \textbf{Calculating weight } $W_q$ 
	\item \textbf{Stopping check: } evaluating conditions to exit the loop
	\end{itemize} 
	\item   \textbf{Output: } ${d =  \Delta\Big(\{h_q,  W_q\}_{q \in Q} \Big)}$
\end{itemize}
\end{tcolorbox}

 Wrapper strategy repeats the Basic training with different parameters to come up with a single decision. 

The operation  $\Delta$ generates new decision $d$ based on hypotheses and their weights obtained in all iterations of basic training. 

Stopping check here checks a specified condition, and if it is true, the procedure exits Wrapper loop. Otherwise, the loop continues with generating new parameters. 

Summarizing, we get a definition of an abduction learner
\begin{tcolorbox}[title = Abduction learner , fonttitle=\fon{pbk}\bfseries ] 
Abduction learner  minimizes an explanation  criterion on a model of Data Analysis Logic  using Basic training with or without Wrapper strategy. 
\end{tcolorbox}

\subsection{The main conjecture} 

Each  popular learning algorithm (the procedures of $k$-NN, Naive Bayes, SVM, hierarchical clustering,  for example) is formulated in unique terms, apparently solves its own type of problem. My conjecture is that they all can be explained as abduction learners,  

\begin{tcolorbox}[title=  Main Conjecture,  fonttitle=\fon{pbk}\bfseries]
In ML every  learning algorithm  is an abduction learner. 
\end{tcolorbox}



\section{Popular learners support the  Main Conjecture} 

To show that a learner is an abduction learner we need to show that, given a training set,  it minimizes an explanation criterion using basic training with or without wrapper loop. 

%
%
%
%
%
%

\subsection{Linkage-based clustering}

The algorithm is also popularly known as hierarchical clustering. 

Clustering can be seen as   modeling a nondeterministic   dependence, where the cluster number is an independent variable, and data vector is a feedback. It can be said that by finding an association of a data point and its cluster number,  we explain the training data as a set of clusters. 

%

 In \cite{Shalev}, a general concept of linkage-based clustering is introduced this way:
\begin{quotation}
	These algorithms proceed in a sequence of rounds. They start from trivial clustering that has each data point in a single-point cluster. Then, repeatedly, these algorithms merge ``closest'' clusters of the previous clustering. $\langle \ldots \rangle$ Input to a clustering algorithm is between-point distance, $d.$ There are many ways of extending $d$ to a measure of distance  between domain subsets (or clusters).
\end{quotation}
The book proposes three ways to evaluate the cluster-distance by the distances between their members:  by minimum distance,   average distance and maximum distance.

The last option clearly contradicts declared goal ``merge `closest'  clusters''. But I will consider it too.

For each round, the training set is a sequence of first order formulas of LDE
$$S = \{\psi( c_i, y_i, \approx), i = 1:m\}.$$
where $c_i \in \mathbb{N}$ is a cluster number of if $i$-th observation, and $y_i$ is the observed data point  of the same observation.  

Denote 
$$C_i = \{y \; | \;   \exists \alpha  (\alpha \in S)  \; \& \;  (y = \bm{y} (\alpha) ) \; \& \;  (i = \bm{x}(\alpha))  \}.$$ the set of data points of the cluster $i.$ Suppose, there are $k$ clusters. 

The purpose of a round is to identify two ``closest'' clusters. 

The notation $h^{ij}, \; j  > i$ will indicate a hypothesis that clusters $C_i, C_j$ ``belong together'', are the best candidates for merging, and $H_k = \{  h^{ij}\; | \; i < j \leq k \}$ denotes  the class of all the hypotheses. 

All the hypothetical   instances of the hypothesis  $h^{ij}$  make  the set 
$$ H^{ij} = \{ \psi( c_i, y, \asymp) \} | \; y \in C_j \}.$$ 

The loss criterion $L(h^{ij}, S) = \rho(C_i, C_j)$ is the cluster-distance between the clusters  $C_i, C_j.$

Now, the learning procedure can be described by the rule:

\begin{tcolorbox}[title = Round of hierarchical Custering, fonttitle=\fon{pbk}\bfseries]
	\begin{itemize} 
		\item \textbf{Fitting}: Generation of  hypotheses  $h  \in H_k$  and evaluation of  the loss criterion $L(h, S)$ on each of them
		\item \textbf{Optimal selection}: Select a hypothesis $h^\prime \in H_k$ with the minimal  value of the loss criterion $L(h, S).$
	\end{itemize}
\end{tcolorbox}

\begin{theorem}
	The linkage based clustering with cluster-distances average, minimum, or maximum is an abduction learner. 	
\end{theorem}
\begin{proof} 
To show that the learner is abduction learner, we need to show that,  the  $L(h^{ij}, S)$  is an explanation criterion for every suggested cluster-distance, and the described procedure is Basic training.

The explanation  criterion  has a single badness rule with alignment  relation 
$$\pi(\alpha_1, \alpha_2) = (\bm{s}(\alpha_1) = \asymp) \& (\bm{s}( \alpha_2) = \approx) \& (\bm{x}(\alpha_1) = \bm{x}(\alpha_2),$$ the same relation as for the $T_{pw}$ badness rule, 
	
	The deviation is defined by the rule
	$$\delta(\alpha_1, \alpha_2) = t(\rho_x(\alpha_, \alpha_2),  \rho_y(\alpha_1, \alpha_2) = \rho_y(\alpha_1, \alpha_2). $$
	It means, the deviation function $t$  in this case also coincides with the deviation function for $T_{pw}.$
	
	The aggregation operation is identified by the type of clustering: minimum, average and maximum.  Minimum and maximum are aggregations by the Theorem \ref{PercentileTheorem}. Averaging is a recursive  aggregation as proven in  Theorem \ref{ERMtheorem}.
	
 Therefore, the criterion  $L(h^{i,j}, S)$ is a badness rule and an explanation  criterion.  

The learning procedure is Basic training without focusing. 

This  proves that each round of linkage-based clustering works as an abduction learner.  
\end{proof}

The learner would agree with the main conjecture not only for the aggregation operations mentioned in the book (average, minimum, maximum), but also for any other aggregation operation.

\subsection{$k$-NN}

 This classification method is intended for the observations with  binary feedback in ${Y = \{0, 1\}}$. The expectation is that in small neighborhood an underlying binary dependence is, mostly, constant. For binary dependencies this condition  is equivalent to being ``explainable'', as discussed in \ref{informal}.   
    The less is the difference between $f$ and the feedback in the training data points close to $x_0$, the better is the explanation. The goal is to find the best explanation out the two hypotheses. 

Small neighborhood is defined by the parameter $k$. Denote $d_k$ distance from the point $x_0$ to the $k$-th closest to $x_0$ data point of the observations $S$. Then $d_k, x_0$ are parameters of the learner. 

The alignment relation  is
$$\pi(\alpha, \beta) =  ( \bm{x}(\alpha) = x_0)  \;\& \;(\bm{s}(\alpha) = \asymp ) \; \& \; (\rho_x(\alpha, \beta) \le d_k) \;\&\; (\bm{s}(\beta) = \approx).$$

The deviation function $t(r_1, r_2) = r_2.$ So, the deviation $\delta(\alpha, \beta) = \rho_y(\alpha, \beta).$

And the aggregation operation is averaging $L(f, S, k, x_0) = L_1(A)$ (see Theorem \ref{AggregationsTheorem}),  where $A$ is  sequence of deviations in arbitrary order.  Thus, the criterion $L(f, S, k, x_0)$ is an explanation criterion.

The procedure of the learner can be described in these steps. 

\begin{tcolorbox}[title = $k$-NN,fonttitle=\fon{pbk}\bfseries]
	\begin{itemize}
	\item \textbf{Focusing:}Defining parameter $d_k$ 
  \item  \textbf{Fitting:} Generating hypotheses $f(x_0) = 0, f(x_0) =1$ and evaluating their error rate $L(f, S, k, x_0)$
	\item\textbf{ Optimal selection: } Selection of the hypothesis  with minimal error rate. 
\end{itemize}
\end{tcolorbox}

The procedure is  the Basic training. 

Thus, $k$-NN is an abduction learner.

\subsection{Two $k$-NN learners with adaptive choice of $k$ } 

The $k$-NN may work, because we presumed the underlying dependence to have, mostly,  small variations of feedback on close data points (section \ref{informal}).  For binary underlying dependence it would mean that it is, mostly, constant in small neighborhoods. 
 
 Optimally, the radius $d_k$  shall be small enough to have  majority of the points in the neigborhood of the same class, and large enough  that random outliers in the finite sample $S(\xi, k)$ do not play much of a role. 

Here I discuss two approaches to select $k$ optimally for every new data point. The first is described in \cite{kNN}, the second is my new algorithm. Both learners find prevalent class $y$ in the focus sample, calculate its frequency  $p_k(y)$  and the error rate ${r_k(y) = 1 - p_k(y)}$ the same as $k$-NN.

Authors \cite{kNN} propose, given a data point $x,$ start with a small $k$ and  gradually increase it  while calculating bias ${t_k(y) = p_k(y) - 0.5}$  of the prevalent class  with every $k$. The procedure stops when  the bias  reaches certain  threshold. If the threshold was not ever reached, they don't output any answer.  So, they search for the smallest neighborhood where the prevalence of one class is above the threshold they picked  beforehand. 

The threshold  they propose to use is: 

$$ 
\Delta(n, k, \delta, c_1 ) = c_1 \sqrt {\frac{   log(n) + log( \frac{1}{\delta})  }{k }  },
$$
where $n$ is size of the training sample, $\delta$ and $c_1$ are some user-selected  parameters, picked before any data analysis. Thus, instead of one parameter, $k,$ the proposed modification require a user to pick 2 parameters with unclear meaning.
 
The learner uses the same criterion as $k$-NN. 

The procedure can be described like this: 

\begin{tcolorbox}[title= Ada k-NN,  fonttitle=\fon{pbk}\bfseries]
	\begin{itemize}
		\item  \textbf{Wrapper Loop: Generating parameter}  $k : = k + 1$  
		\begin{itemize}
		\item \textbf{Basic training } with parameter $k$
		\begin{itemize} 
			\item \textbf{Focusing:} Transformation $S \rightarrow S(x_0, k)$ 
			\item \textbf{Fitting: } Generation of two constant hypotheses and evaluation of their loss $L(f, S, k, x_0) ) $  
			\item\textbf{Optimal selection: } Outputs the constant hypothesis $f^\prime$   with minimal loss. . 
			\end{itemize}
		  \item \textbf{Stopping  check:} $\big( L(h^\prime, S(x_0, k)) > \Delta(S,k,\delta, c_1) \big)$ or $(k = n)$ 
		  \end{itemize} 
		\item \textbf{Output: } If $k < n$, output $f^\prime$ as decision. Otherwise, refuse to output a decision. 
	\end{itemize}
\end{tcolorbox}

Thus, the learner performs the  basic  training of the original $k$-NN with wrapper for parameter selection,  corroborating the main conjecture.

This learner is developed within the statistical learning paradigm, where the training set is expected to be arbitrary large. As $n$ increases, so does the threshold $\Delta(n, k, \delta, c_1 )$. Therefore, the selected  value $k, $  the size of the focus training set,  will go to infinity with $n.$ And thus, by the law of  large numbers, the solution will converge asymptotically to the expectation of the class in the given neighborhood. At the same time, the ratio of $k$ to $n$ is expected to decrease, thus the size the $k$-neighborhood will tend to 0. If the distribution is continuous in $x,$ then the leaner will likely find the solution as $n$ tends to infinity. 

The issue here is that $n$ is not going to infinity or anywhere.  For a fixed $n,$ the learner favors smaller $k$, where the evaluation of prevalent class is subject to random fluctuations caused by small sample.

To alleviate this issue, I propose an  alternative approach which uses Hoeffding inequality ( see, for example, \cite{Shalev}) to select $k$.   

The Hoeffding inequality can be written as
\begin{equation}\label{Hoef}
P[\; | \, p - E\,   | > t\;  ]  \leq 2 \;exp( - 2 k \, t^2) , 
\end{equation}
where $p$ is observed frequency of an event,  $E$ is  the expected frequency (probability) of the same event, and $t$ is an arbitrary threshold, and $k$ is the sample size. 

Suppose, $p$ evaluates observed frequency of class 1 (rate of the class 1 among the neighbors), $E$ is the probability of the class 1 in the neighborhood of a given point.  
 If $p$ is above 0.5, then observations of the class 1 prevail,  and  we pick hypothesis 1 out of two.  Otherwise, the we pick hypothesis 0. 

Let $t = | \,0.5 -p \,|. $ If $ | \, p - E\,   | > t$  the  expected prevalent class is different from the observed prevalent class.  If it is the case,  we selected the wrong hypothesis. In this case, the right side of the inequality gives us an upper limit of probability  that we picked the prevalent class wrong. 

 For selection of $k$ we use the weight, calculated as the right part of  (\ref{Hoef}) : 
$$W(y, S, k)  = 2 \cdot exp( - 2 \; k \; |\,p - 0.5\,| ^2 ). $$
Obviously, the larger is $k$, and the further is the frequency $p$ from $0.5$, the lower is the weight. The weight will serve well for the selection of the parameters $k$, because we  need to find the neighborhood where $p$ is far from uncertainty, $0.5$, yet, the size of the neighborhood is not too small. 

Here is the description of the learner's procedure for the given data point $x$.

\begin{tcolorbox}[title= Hoeffding k-NN, fonttitle=\fon{pbk}\bfseries]
	\begin{itemize}
		\item  \textbf{Generation of parameter} $k := k + 1$ 
		\begin{itemize}
			\item \textbf{Basic training:}
			\begin{itemize} 
				\item \textbf{Focusing:} Select focus training set $Q_k(x)$ of $k$ observations with data points closest to $x.$
				\item \textbf{Fitting: } Evaluate error rate $r_k(c)$  of  hypotheses $c \in \{0, 1\}$ in $Q_k(x)$ 
				\item\textbf{Optimal selection: } Select the hypothesis $c^\prime(k)$  with minimal error rate $r_k(c^\prime(k))$. 
			\end{itemize}
			\item \textbf{Calculating weight} $W(x, S, k).$ 
				\item \textbf{Stopping check:} $k = n-1$
		\end{itemize} 
			\item \textbf{Output: } 
 	 $k^\prime = \arg \min W(x, S, k);$ output $c^\prime (k^\prime). $
	\end{itemize}
\end{tcolorbox}

Thus, this  learner is an an abduction learner as well. 

\subsection{Decision trees}

For this learner, the features  are expected to be ``ordinal'': every  feature has finite number of ordered  values; there are no operations on feature values.  The feedback of observations is binary.  Again, the assumption is that the underlying dependence is mostly constant in small neighborhood. Here, the size of the neighborhood is not set up a priory. The algorithm finds maximal homogeneous neighborhoods for the best explanation of observations.

 The learner starts with whole domain, splits it in two subdomains by a value of some feature. Then, the procedure is repeated for every of the subdomains  until a subdomain called "leaf" is reached. The decision is selected for this subdomain. The navigation over the tree of subdomains continues until some stopping criterion is reached. 
 The algorithm has a precise rule for generating the parameters of the next subdomain based on the previous trajectory and the obtained results.

 There are two criteria of  a leaf:
 \begin{enumerate}
 	\item Number of observations  in the subdomain is below a threshold $N$.
 	\item Percentage of observations  of the prevalent class in the subdomain is above the threshold $q$ (it is homogeneous).
 \end{enumerate}

In each subdomain the procedure selects one of two binary hypotheses with minimal error rate. It is  easy to see that badness rule for selection of the hypothesis is $T_{pw}$ \ref{PWR}. We can denote $L_{tw}(h, S)$ the explanation criterion of the learner. 

The procedure may be described as a Basic  training  with Wrapper strategy:
\begin{tcolorbox}[title=Decision Tree, fonttitle=\fon{pbk}\bfseries]
\begin{itemize}
	\item \textbf{Generating parameters} $g$ of the next subdomain
	\begin{itemize}
			\item \textbf{Basic training:}
			\begin{itemize}
		\item \textbf{Focusing:} select  subdomain $G(g)$ and subset of the training set $S(g)$ with parameters $g$ 
		\item \textbf{Fitting}: Generate hypotheses $ h \in \{0, \; 1\}$ on $S(g)$  and evaluate their loss criteria $L_{pw}(h, S(g)).$
		\item \textbf{Optimal selection}: select the hypothesis $d(g)$ with minimal  value of $L_{pw}(h, S(g)), h \in  \{0, \; 1\}.$
		\end{itemize} 	
    \item \textbf{Calculating weight:} If $G(g)$ is a leaf, $W(g) = 1$, otherwise $W(g) = 0.$
	\item \textbf{Stopping check} End of tree
 \end{itemize}
\item  \textbf{Output:} For each $g: W(g) = 1$ output  $d(g)$ as decision on $G(g)$. 
\end{itemize}

\end{tcolorbox}

For the points outside of any leaf the decision is not defined.

Therefore, the Decision tree is  an abduction learner as well.

\subsection{Naive Bayes}

 The underlying dependence has $n$ independent variables and binary feedback. The learner works as if it deals with  nominal  data: the only relationship between data points is equivalence. 

The procedure defines decision function on one data point at the time. 
For a given data point $z= \langle z_1, \ldots, z_n \rangle$ the procedure selects $n$ subsets of the training set.  Subset $S_j$ includes all the observations with $j$-th variable  equal $z_j.$ For each subset $S_j$, the learner evaluates error rate $e(c, S_j)$    of each hypothesis $c \in \{0, 1\}.$ Then for each hypothesis   it calculates criterion
$$\Delta(c, S)  =   \prod_j ( 1- e(c, S_j)).$$ The learner  selects a hypothesis  with the maximal value of the criterion. 

The explanation   criterion for this learner has $n$ badness rules as well as a functional $\psi$ to aggregate values of all $n$ badness criteria. 

The $i$ badness rule for a hypothesis $c$ is
$$\pi_i( \alpha,  \beta) = (\bm{s}(\alpha) = \asymp) \& (\bm{s}(\beta) = \approx) \& (\bm{x}_i(\alpha)  = \bm{x}_i(\beta)).$$
$$t(r_1, r_2) = r_2$$
$$ \delta_i(\alpha, \beta) = \rho_y(\alpha, \beta).$$
Aggregation operation in each badness rule is recursive aggregation averaging, $L_1(A)$ from the Theorem \ref{AggregationsTheorem}. 

A combining  operation to combine the results of all  badness rules is 
$$\psi(A) = 1 - \prod_i(1 - a_i).$$

It is easy to see that the operation $\Psi$  is both monotone and order insensitive (see \ref{AggAxioms}) . 

Thus, the loss criterion of Naive Bayes is an explanation criterion. 

Now the procedure of the learner with given data point $z = \langle z_1, \ldots, z_n \rangle$  may be described by the rule 

\begin{tcolorbox}[title = Naive Bayes, fonttitle=\fon{pbk}\bfseries]

\begin{itemize}
	\item \textbf{Fitting}: generating hypotheses $c \in \{0, 1\}$ and calculating their loss criterion $\Delta(c, S)$
				\item \textbf{Optimal selection} Select a hypothesis with minimal  $\Delta(c, S).$				
   \end{itemize}
\end{tcolorbox}

On the Fitting step,  the procedure calculates $n$ badness values (corresponding $n$ badness rules)  for each hypotheses, then aggregates  these values into the abduction criterion for a given hypothesis.  Interesting that calculation of each individual badness on this way  requires focusing.  However, as always, a process involved in calculation of a criterion is not reflected in the scheme of the learner. 

This proves that Naive Bayes Naive Bayes is an abduction learner. 

A product in the aggregation of the badness values  is chosen  in Naive Bayes   because  it is  sensitive to the low frequencies of a class: if some  value $1 - e(c, S_j)$ is  close to 0, the product will be affected much more than the sum of the frequencies, for example. If some feature value almost never happens in a given class $c$, the hypothesis $c$ will have no chance of being selected, regardless of other feature values of $z$.  It justifies choice of product for aggregation.

The loss criterion of Naive Bayes  is  traditionally interpreted as evaluation of posterior probabilities with ``naive''  assumption that the features  are independent. There are several issues with this narrative. 

First, it works for only one learner: if  NB  learner is based on  naive idea about Bayes rule, other learners would need different foundations.  

Another issue is that  it creates an impression that the learner needs  an improvement,  is not sophisticated enough. 

I hope, I demonstrated that interpretation of the learner as   ``naive'' and ``Bayesian'' misses the point. The procedure   is driven by its specific data type, and it is explained as performing AGC inference, the same as majority of other learners. 

\subsection{Logistic Regression}

This learner assumes the features are continuous, the feedback of the observations is binary, but the feedback of the decision is continuous: so, it is required some rounding up in each data point. The decision is defined on the domain  $\chi. $ The procedure of generating the hypotheses is not specified. 

The class of functions associated with logistic regression is 
$$F =\left \{ \frac{ 1} {1 +  exp(- \langle w, x\rangle)  } \right \}.$$
The functions have values in the interval $(0, 1).$

The learner minimizes criterion

$$\Delta(f, S) = \frac{1}{m} \sum_{s \in S} \log\Big( |y(s) - f\big(x(s) \big)|  \Big).$$

The explanation criterion  contains only one badness rule. 
$$\pi(\alpha, \beta) = \pi_{pw},$$ 
see \ref{PWR}.
The contradiction degree  is
$$\delta(\alpha, \beta) = log( \rho_y(\alpha,\beta)).$$
The function is isotone by $\rho_y(\alpha_1, \alpha_2)$ and does not depend on $\rho_x(\alpha_1, \alpha_2).$ 

The proper   aggregation  operation  is averaging ($L_1(A)$). 
It is obvious that this badness rules defines the criterion $\Delta(f, S).$

So, the logistic regression supports the main conjecture as well. 

The issue with this learner is that, the same as ERM, it does not take into account the similarity of feedback in close (not identical) data points,  and, therefore, has a tendency of overfitting.

\subsection{Linear SVM for classification}

All the previous learners belong to  machine learning ``folklore''. Their authors are not known, or, at least,  not  famous. 

SVM is one of the first learners associated with a known author: it  is  invented by  V. Vapnik. His earliest English publications  on this subject appeared in early nineties \cite{Vapnik1}, \cite{Vapnik2}.

Let us start with linear SVM for binary classification. 
The observations $$S = \{\beta_i, i= 1:m\}$$  have two class labels: $\{-1, 1\}$  with data points $x \in \bb{R}^n. $

The class of hypotheses $F$ consists of  linear functions $f(x) $ with $n$ variables.  For a $f \in F,  f(x) = x^T \beta  + \beta_0.$ denote $\bm{w}(f) = \beta, \bm{b}(f) = \beta_0.$

The problem is formulated as minimization of the criterion
\begin{tcolorbox}[title = Linear SVM, fonttitle=\fon{pbk}\bfseries]
\begin{align} \label{SVM}
	L(f, S, \xi)  =   \alpha \, \| \bm{w}(f)  \|^2 + \frac{1}{m} \sum_{\beta \in S }^m \xi(\beta) \\
 \text{s.t. }  \text{ for all }  \beta \in S,  \; \; \bm{y}(\beta) \cdot f(\bm{x}(\beta)) \ge 1 - \xi(\beta)\; \text{ and } \; \xi(\beta)  \ge 0. \label{conditions} 
\end{align}
\end{tcolorbox}

The criterion looks intimidating, but it may be simplified though. For this, we want to switch to narrower class of functions, which shall contain all the same decisions. 

The observations $\beta \in S$ satisfying condition $\bm{y}(\beta) \cdot f(\bm{x}(\beta)) > 0. $  are considered correctly classified by the function $f$. Denote $S^\oplus(f)$ all correctly classified observations by the function $f, $ and $S^\ominus(f) = S \setminus S^\oplus(f)$ the rest of the observations. 

Let us consider all the functions $f \in F$ such that 
$ S^\oplus(f) \neq \emptyset $ and 
$$\min_{S^\oplus(f) }| f(\bm{x}(\beta))| = 1.$$
Denote this class of function $F^\prime(S).$  The class $F^\prime(S)$ is not empty. Indeed, if for some $f, f  \not \equiv 0,$ $S^\oplus(f) = \emptyset$, then, $S^\oplus( -f) = S$.  If  $$q = \min_{S^\oplus(f) }| f(\bm{x}(\beta))| \neq  1, $$ then 
the function $f^\prime = \frac{1}{q} f$ satisfies the condition $$\min_{S^\oplus(f) }| f^\prime(\bm{x}(\beta))| = 1.$$

The last consideration implies that if $f$ is the decision of the problem, 
then the problem has a decision  $f^\prime$ in the class $F^\prime(S)$ with the same set of correctly recognized observations $S^\oplus(f^\prime) = S^\oplus(f).$

Therefore, we can restrict the search for a decision in the class $F^\prime(S)$ only.

\begin{theorem}
	The linear  SVM classification problem minimizes the loss criterion
	$$L_{svm}(f, S) = \alpha \| \bm{w}(f) \|^2 + \frac{1}{m} \sum_{\beta \in S^\ominus(f)} |\bm{y}(\beta) - 	f(\bm{x}(\beta)) |,$$
for $f \in F^\prime(S). $
\end{theorem}
\begin{proof}

 The conditions (\ref{conditions})  can be rewritten as $\forall \beta, \beta \in S:$
 \begin{equation}\label{cond}
 	\begin{cases}
 		\xi(\beta) \geq  1  - \bm{y}(\beta) \cdot f(\bm{x}(\beta))   \\
 		\xi(\beta) \geq 0.
 	\end{cases}
 \end{equation}
 or  
 $$\xi(\beta) \ge \max \big\{  1 - \bm{y}(\beta) \cdot f(\bm{x}(\beta)) , \; 0 \big\}.$$

  The values $\xi(\beta), \beta \in S$  do not depend on each other, so the minimum of their sum  is achieved when every variable $\xi(\beta)$ equals its lowest possible value. Let us find these lowest values for $\xi(\beta)$ depending on  if $\beta \in  S^\oplus(f)$ or  $\beta \in S^\ominus(f).$
 
 If $\beta \in S^\oplus(f),$ 
 $$\bm{y}(\beta) \cdot f(\bm{x}(\beta)) = |f(x(s))|. $$ 
 
 By definition of $F^\prime(S), $ $|f(x(s))| \geq 1.$
 Then $$\xi(\beta) \geq \max \big\{  1 - \bm{y}(\beta)\cdot f(\bm{x}(\beta)) , \; 0 \big\} = 0.$$ 
 In this case, the lowest possible value for $\xi(\beta)$ is 0.

If $\beta \in S^\ominus (f),$ $$\bm{y}(\beta) \cdot f(\bm{x}(\beta)) = - |f(\bm{x}(\beta))|.$$ Then 
$$\xi(\beta) \geq \max \big\{  1 - \bm{y}(\beta)\cdot f(\bm{x}(\beta)) , \; 0 \big\} = 1 + |f(\bm{x}(\beta))|.$$
In this case, the lowest possible value for $\xi(\beta)$ is $1 + |f(\bm{x}(\beta))|.$

So,

\begin{equation}
\min_\xi \frac{1}{m} \sum_S \xi(\beta) = \sum_{\beta \in S^\ominus(f)} ( 1 + |f(\bm{x}(\beta))| ).
\end{equation}

We still need to prove that for $\beta \in S^\ominus(f)$
$$1 + |f(\bm{x}(\beta))|  =  |\bm{y}(\beta) - f(\bm{x}(\beta))|. $$

Let us take $\beta \in S^\ominus(f).$ 
If $\bm{y}(\beta) = 1, $ then $f(\bm{x}(\beta)) < 0$ and $|f(\bm{x}(\beta))| = - f(\bm{x}(\beta)).$ So, 
$$(1 + |f(\bm{x}(\beta))| ) = 1 - f(\bm{x}(\beta)) = | \bm{y}(\beta) - f(\bm{x}(\beta)) |.$$
If $\bm{y}(\beta) = -1, $ then $f(\bm{x}(\beta)) > 0$ and $|f(\bm{x}(\beta))| = f(\bm{x}(\beta)).$ So,
$$1 + |f(\bm{x}(\beta))| = 1 + f(\bm{x}(\beta)) = -\bm{y}(\beta) + f(\bm{x}(\beta)) = |\bm{y}(\beta) - f(\bm{x}(\beta))|.$$

\end{proof}

Now  I need to show that to prove that $L_{svm}(f, S)$ is an explanation criterion.


The distance between the feedback of observations and the function value is defined  here by the rule

\[\rho_y(\alpha_1, \alpha_2) = 
\begin{cases}
	0, & \text{ if } \bm{y}(\alpha_1) \cdot \bm{y}(\alpha_2) \ge 0\\
	|\bm{y}(\alpha_1) -\bm{y}(\alpha_2)|, & \text{otherwise}. 
\end{cases}
\]

The only badness rule  coincides with point-wise rule $T_{pw}$. 

The formula  $Z(f) = \|w(f)\|^2$ of the criterion  is a regularization component: $w(f)$ is the gradient of the hypothesis $f$, and  $\|w(f)\|^2$ is the square of its norm. 
The functional $\Psi$ which combines  values of these two criteria is $\Psi(x_1, x_2) = \alpha x_1 + x_2.$

This confirms that linear SVM for classification minimizes an explanation criterion.

\subsection{Linear Support vector regression}

The learner ( \cite{Hastie}) minimizes criterion 

$$L_{svr}(f, S) = \sum_{i = 1}^m V_\epsilon\big(\bm{y}(\beta_i) - f(\bm{x}(\beta_i))\big) + \lambda \|\bm{w}(f)\|^2, $$
where 
\[
V_\epsilon(r) = 
\begin{cases}
	0, & \text{if  } |r| < \epsilon\\
	|r| - \epsilon, & \text{otherwise}
\end{cases}
\]
and $S = \{\beta_1, \ldots, \beta_m\}.$

The class of hypothesis  is the class of all linear functions. 

The second component of the loss criterion is regularization, the same as in the SVM.  The distance between feedback of an observation and the value of a hypothesis is defined through the function $V:$
for $\alpha_1 = (\asymp(\ph(x_1) = y_1) ), \alpha_2 = (\approx( \ph(x_2) = y_2))$
$$\rho_y(\alpha_1, \alpha_2) = V(y_1  -  y_2).$$

Then the AGC criterion scheme here  coincides with the scheme for  linear SVM for classification.  

So, the linear support vector regression supports the main conjecture as well.

\subsection{Support Vector Regression with Kernels}

The learner is defined as in  (\cite{Hastie}). Suppose we have a set of basis functions $H= \{h_i(x), i = 1,\ldots, k\}.$
We are looking for a hypotheses 
$$f(x) = \sum_{i=1}^k w_i h_i(x)  + b.$$
The loss criterion  used here is
$$L(f, S) = \sum_{i= 1}^m V\big( \bm{y}(\beta_i) - f(\bm{x}(\beta_i) \big)  + \lambda \|\bm{w}(f)\|^2,$$
where 
\[
V(r) = 
\begin{cases}
	0, & \text{if  } |r| < \epsilon \\
	|r| - \epsilon, & \text{otherwise}.
\end{cases}
\]

Here the transformation  $x \rightarrow \langle h_1(x), \ldots, h_k(x)\rangle$ from a $n$-dimensional space $R^n$ into $k$-dimensional space $H(x)$ may be called focusing.
Then the problem is reduced to solving a linear SVM regression in the transformed space. 
Thus, SVR with kernel supports the main conjecture as well.

\subsection{Ridge Regression} 

The learner  finds the solution in the class of all linear hyperplanes $F = \{f: \; \, f= wx + b\}, $ and it has the criterion

$$ L_{rr}(f, S) =  \alpha \|\bm{w}(f)\|^2 + \frac{1}{m} \sum_{\beta \in S}  (f(\bm{x}(\beta)) 	- \bm{y}(\beta))^2.$$

The first component of the loss criterion is regularization component, the same as in SVM, SVR. 

Unlike SVR, in this case the distance on $Y$ is $\| y_1 - y_2\| = (y_1 - y_2)^2.$ While SVR does not count small errors, Ridge regression counts all errors, but small errors have low influence because of the square in the norm. 

The explanation criterion  for Ridge Regression is not any different from the criterion for SVR. The only difference between the learners is in the interpretation of Data Explanation Logic: the distances on $Y$ are different. 

Thus,  Ridge regression corroborates the main conjecture too.

\subsection{Neural Network  (NN) }

Let us consider single hidden layer NN for two class classification as it is described in \cite{Hastie}.

First, the learner transforms  $n-$ dimensional metric space of inputs $\bb{R}$ into $k$-dimensional space $\bb{Z}$ using non-linear transformation;
$$Z_i(x) = \delta( g_i(x) ), i = 1, \ldots, k,$$
where $\delta(r)$ is delta function and $g_i$ are linear functions.  Denote $\bm{z}(x)$ the vector with coordinates $\langle Z_1(x), \ldots, Z_k(x) \rangle.$

Then, for each class $c \in \{0, 1\}$, the learner builds  linear voting function  $f_c(\bm{z}(x)).$ 
Denote $G = \{g_1, \ldots, g_k\}$, and $F= \{f_0, f_1\}.$ 

For each $x \in \bb{R}^n$ the class is selected as
$C(x,  G, F) =  \arg \max_c  f_c(\bm{z}(x)).$ 

The learner uses the  loss criterion 
$$L_{nn}(G, F, S) =  \sum_{\beta \in S} ( \bm{y}(\beta) - C(\bm{x}(\beta), G, F) ).$$
It is obvious that the loss criterion is a point-wise badness rule $T_{pw}.$

The learner optimizes simultaneously  parameters of the functions $G$ and $F.$ For selection of parameters of these functions the learner uses gradient descent, which is called ``back propagation'' in this case. The learner uses some additional stopping criterion.  

So, the procedure does not have a focusing stage. If calculates loss for a given set of parameters,  evaluates  gradients by  each parameter, and then updates parameters based on the gradients. After the stopping criterion is achieved, the algorithm outputs the decision with the lowest loss criterion. 

The procedure has only two types of steps: 
\begin{enumerate}
\item fitting, which includes
\begin{itemize} 
\item  generation of the $C(x, G, F)$ hypothesis based on previous value of loss criterion and gradients 
\item evaluation of loss criterion $L_{nn}(G, F, S).$ of the current hypothesis 
\end{itemize}
\item optimal selection: selection of the hypothesis with the lowest loss criterion.  
\end{enumerate}

Thus, NN is also an abduction learner.

\subsection{$K$ Means Clustering}

The learner is different from hierarchical clustering in that it does not combine clusters, rather, for each observation, it chooses the proper cluster. It is assumed that the distance on the domain of  data points is Euclidean. 

Here is the description of the learner from \cite{Hastie}. 
	\begin{enumerate}
		\item Given the current set of means of clusters $M = \{m_1, \ldots, m_K\}$, each observation is assigned to the cluster with the closest mean.
		\item The rounds of assignment of all observations are repeated until clusters do not change.
	\end{enumerate}

The  learning happens when we search for the cluster for the given observed data point.  Denote $C(x)$ the assignment of a cluster to a data point $x$. Given the set of observed  data points $S_x = \{x_1, \ldots, x_m\},$  $K$ clusters  with cluster centers $M$ of the sizes $\{l_1, \ldots, l_K\}$
the procedure assigns a new class to an observed data point  to  minimize sum of all pairwise distances within each cluster
\begin{align}
W(C, S) & = \frac{1}{2} \sum_{k=1}^K \sum_{C(\xi) = k} \sum_{C(\zeta) = k} \| \xi - \zeta\|^2 \label{1line}\\ 
        & = \sum_{k=1}^K l_k \sum_{C(\xi) = k}  \| \xi - \overline{x}_k\|^2,
\end{align}
where $\xi, \zeta \in S$, and $\overline{x}_k$ is mean of the $k$-th cluster. 
I use the form (\ref{1line}) to prove that the learner agrees with the main conjecture. 

Denote  $x_0$ a data point, $x_0 \in S_x,$ which we need to assign a cluster on this step.

As in the case of hierarchical clustering,  we consider underlying dependence $\ph$ as a function from cluster index $k$ to the observed data point $x.$  There are $K$ hypotheses $H(x_0) = \{h_1, \ldots, h_K\}.$ Each hypothesis $h_i$ has a single hypothetical case $\psi( x_0, i , \asymp).$ 

 We assume, before current run of the learner,  the clusters are already assigned to each observed data point besides $x_0$. So, the run starts with the training set having observations $$S(x_0) = \{\beta \in S: \bm{x}(\beta) \neq x_0\}.$$

Let us define an explanation criterion equivalent with \ref{1line}. There is only one badness rule  with the alignment relation 
$$\pi(\alpha_1, \alpha_2) = ( \bm{x}(\alpha_1) = \bm{x}(\alpha_2)),$$
which says that we evaluate contradiction degree  for each pair of formulas with the same cluster, regardless if it is an observation or a hypothetical instance. The alignment relation  is symmetrical, therefore for each pair of formulas $\alpha_1, \alpha_2$  which satisfies the condition, the pair $\alpha_2, \alpha_1$ satisfies the condition as well. In effect, every pair is counted twice.  

The deviation  function is
$$t(r_1, r_2) = \frac{1}{2} r_2^2$$ 
and the deviation is calculated by formula 
$$\delta(\alpha_1, \alpha_2) =  t(\rho_x(\alpha_1, \alpha_2), \rho_y(\alpha_1, \alpha_2) =  \frac{1}{2}\rho_y(\alpha_1, \alpha_2)^2 $$

For the recursive  aggregation we use the averaging $L_1(A)$ again. 

The learner generates all hypotheses $H(x_0)$, evaluates the explanation criterion for each of them and selects the hypothesis with the lowest value of the  criterion. This is Basic training. Thus this learner corroborates the main conjecture as well.

\section{Conclusions}

Peirce considered  learning from experimental data as  ``data explanation''.  Pragmatism,  which he introduced,   informally  describes the logic behind this process   as abduction inference.    ML is an automated learning from data for practical applications.   If Peirce is right, ML could be understood  as an automated  abduction inference. And  indeed I demonstrated here that it is the case. 

The  criterion of data explanation  is formalized here within Data Analysis Logic I defined. No statistical concepts were used in this formalization. 
The process of the criterion minimization is described as two standard procedures: Basic training and Wrapper. Combination of an explanation criterion and the procedure of its minimization makes an abduction learner. My conjecture here is that every ML learners is an abduction learner. This conjecture is corroborated on 13 most popular learners for classification, clustering and regression. Thus, ML may be called an automated abduction indeed. 

The approach has important advantages over commonly accepted statistical learning paradigm. It  allows to understand

\begin{enumerate}
	\item real life learners in their variety; their differences and common features;
	\item conditions for learning form finite data sets
	\item common structure of  classification, regression and clustering problems and their algorithms
	\item  regularization as an aspect of an explanation criterion. 
\end{enumerate}

The future development of this approach may include
\begin{itemize}
	\item Logical understanding of testing as necessary part of learning process.
	\item  Logical approach toward data analysis beyond learning
	\item Automated algorithm selection. 
	\item Understanding  of other learning tasks, such as survival analysis, ranking  \cite{BipartiteSapir},   optimal choice \cite{OptChoice}.
    \item Development of new algorithms based on the proposed ideas. One example of such novel algorithm, adaptive $k$-NN,  is introduced here. 
    \end{itemize}

\bibliographystyle{plain} 
\bibliography{Smoothing} 

\begin{thebibliography}{10}

\bibitem{kNN}
A.~Balsubramani, S.~Dasgupta, and Y.~Freund.
\newblock An adaptive nearest neighbor rule for classification.
\newblock In {\em 33rd Conference on Neural Information Processing Systems
  (NeurIPS 2019), Vancouver, Canada.}, 2019.

\bibitem{Vapnik2}
B.E. Boser, I.~M. Guyon, and V.N. Vapnik.
\newblock A training algorithm for optimal margin classifiers.
\newblock In {\em COLT '92: Proceedings of the fifth annual workshop on
  Computational learning theory}, 1992.

\bibitem{Vapnik1}
C.~Cortes and V.~Vapnik.
\newblock Support vector networks.
\newblock {\em Machine Learning}, 20:273 -- 297, 1995.

\bibitem{AbductionComplexity}
Thomas Eiter and Georg Gottlob.
\newblock The complexity of logic-based abduction.
\newblock In {\em Tenth Symposium on Theoretical Aspects of Computing (STACS),
  LNCS 665}, pages 70--79. Springer, 1993.

\bibitem{Hastie}
T.~Hastie, R.~Tibshirani, and J.~Friedman.
\newblock {\em Elements of statistical learning}.
\newblock Springer, 2009.

\bibitem{StatTheory}
Ulrike~von Luxburg and Bernhard Sch¨olkopf.
\newblock Statistical learning theory: models, concepts and results.
\newblock In Dov~M. Gabbay, Stephan Hartmann, and John Woods, editors, {\em
  Handbook of the History of Logic. Volume 10: Inductive Logic}, pages
  651--706. Elsevier BV, 2009.

\bibitem{Stability}
Sayan Mukherjee, Partha Niyogi, Tomaso Poggio, and Ryan Rifkin.
\newblock Statistical learning : stability is sufficient for generalization and
  necessary and sufficient for consistency.
\newblock {\em Advances in Computational Mathematics}, 25:161--193, 2006.

\bibitem{PeirceV1}
C.S. Peirce.
\newblock {\em The Essential Peirce, Volume 1}.
\newblock Indiana University Press, 1992.

\bibitem{Pierce}
C.S. Pierce.
\newblock Abduction and induction.
\newblock In {\em Philosophical writings of Pierce}, pages 150--156. Routledge
  and Kegan Paul Ltd, 1955.

\bibitem{OptChoice}
M.~Sapir.
\newblock Optimal choice : New machine learning problem and its solution.
\newblock {\em International Journal of Computational Science and Information
  Technology}, 5:1 -- 9, 2017.

\bibitem{BipartiteSapir}
Marina Sapir.
\newblock Bipartite ranking algorithm for classification and survival analysis.
\newblock {\em arXiv}, 2011.

\bibitem{Papaya}
Marina Sapir.
\newblock {\em Papaya Orchard: Comedy in one act}, 2018.
\newblock
  \url{https://www.academia.edu/35254962/_Papaya_Orchard_Comedy_in_one_act}.

\bibitem{Shalev}
Shai Shalev-Shwartz and Shai Ben-David.
\newblock {\em Understanding Machine Learning.}
\newblock Cambridge University Press, NY, 2014.

\bibitem{VapnikBook}
V.~N. Vapnik.
\newblock {\em The nature of statistical learning theory}.
\newblock Springer - Verlag, 1995.

\end{thebibliography}

\end{document}